\documentclass[letterpaper, 10 pt, conference]{ieeeconf}  % Comment this line out
\IEEEoverridecommandlockouts                              % This command is only
\usepackage{amsmath} % assumes amsmath package installed
\usepackage{amsfonts}
\usepackage{amssymb}
\usepackage{xcolor}
\usepackage{bbold}
\usepackage{graphicx}
\usepackage{ulem}

\newtheorem{lemma}{Lemma}

\title{\LARGE \bf
Fusion of Indirect Methods and Iterative Learning for Persistent Velocity Trajectory Optimization of a Sustainably Powered Autonomous Surface Vessel
}

\author{Kavin M. Govindarajan, Devansh R Agrawal, Dimitra Panagou, and Chris Vermillion% <-this % stops a space
\thanks{Supported by NSF Award Number 2223844.}% <-this % stops a space
\thanks{Kavin M. Govindarajan is a Graduate Research Assistant with the Department of Robotics, University of Michigan, Ann Arbor, MI 48109, USA.     {\tt\small kmgovind@umich.edu}}%
\thanks{Devansh R Agrawal is a Graduate Research Assistant with the Department of Aerospace Engineering, University of Michigan. Dimitra Panagou is an Associate Professor with the Department of Robotics and the Department of Aerospace Engineering, University of Michigan, Ann-Arbor, MI, 48109, USA. {\tt\small devansh@umich.edu, dpanagou@umich.edu}}%
\thanks{Chris Vermillion is an Associate Professor with the Department of Mechanical Engineering, University of Michigan, Ann-Arbor, MI, 48109, USA. {\tt\small cvermill@umich.edu}}%
}

\begin{document}

\maketitle
\thispagestyle{empty}
\pagestyle{empty}

%%%%%%%%%%%%%%%%%%%%%%%%%%%%%%%%%%%%%%%%%%%%%%%%%%%%%%%%%%%%%%%%%%%%%%%%%%%%%%%%
\begin{abstract}

In this paper, we present the methodology and results for a real-time velocity trajectory optimization for a solar-powered autonomous surface vessel (ASV), where we combine indirect optimal control techniques with iterative learning. The ASV exhibits cyclic operation due to the nature of the solar profile, but weather patterns create inevitable disturbances in this profile. The nature of the problem results in a formulation where the satisfaction of pointwise-in-time state of charge constraints does not generally guarantee persistent feasibility, and the goal is to maximize information gathered over a very long (ultimately persistent) time duration. To address these challenges, we first use barrier functions to tighten pointwise-in-time state of charge constraints by the minimal amount necessary to achieve persistent feasibility. We then use indirect methods to derive a simple switching control law, where the optimal velocity is shown to be an undetermined constant value during each constraint-inactive time segment. To identify this optimal constant velocity (which can vary from one segment to the next), we employ an iterative learning approach. The result is a simple closed-form control law that does not require a solar forecast. We present simulation-based validation results, based on a model of the SeaTrac SP-48 ASV and solar data from the North Carolina coast. These simulation results show that the proposed methodology, which amounts to a closed-form controller and simple iterative learning update law, performs nearly as well as a model predictive control approach that requires an accurate future solar forecast and significantly greater computational capability.

\end{abstract}

%%%%%%%%%%%%%%%%%%%%%%%%%%%%%%%%%%%%%%%%%%%%%%%%%%%%%%%%%%%%%%%%%%%%%%%%%%%%%%%%
\section{INTRODUCTION}
%% Outline
%% - Why is ocean data useful?
%% - What observation methods currently exist?
%% - Why do we need persistent missions?
%% - Introduce ASVs as a persistence solution
%% - Reference old work
%% - Reference optimal control literature that explores multi-agent persistence... but doesn't have energetic constraints (https://ieeexplore.ieee.org/abstract/document/6332481/?casa_token=Hb_HSfuPqaIAAAAA:zDfDWarOh3EvKQF13Z3Ep-gCvaUCiI5nSSih5YhJyesOF6uThrnzGtYCvhksSMkllXyXlpp5xQ)
%% - Motivate that in this work, we seek to solve the 1-D velocity trajectory optimization problem subject to energetic constraints and demonstrate a solution that is simpler to implement than our MPC from last paper

The collection of oceanographic data is useful for many applications: Characterizing ocean currents is useful for identifying candidate marine energy-harvesting sites \cite{naik2024integrated}, surface temperature data is useful for weather forecasting \cite{huda2023machine}, and salinity data is useful for better understanding processes driving climate change \cite{seim2022overview}. To collect such data, several observational tools are currently employed, including moored sensor platforms \cite{adcp}, boat-mounted platforms \cite{muglia_transect}, high-frequency radar \cite{radar}, and undersea gliders (see \cite{glider}, \cite{glider2}). However, these tools often result in sparse and/or short-duration measurements.

To collect spatially granular, long-duration oceanographic data, the aforementioned existing sparse and/or short-duration measurements must be supplemented by granular, persistent, and autonomous observations. Renewably powered marine mobile robots, such as sailing drones \cite{gentemann2020adaptively} and solar-powered autonomous surface vessels (ASVs) \cite{seatrac}, aim to satisfy those needs. These systems can operate persistent missions while achieving the mobility and granular data collection required for certain missions.

In this work, we consider the long-duration operation of a solar-powered autonomous surface vessel (ASV) (pictured in Fig.\ref{fig:seatrac}) in an ocean environment. We treat distance traveled as a proxy for information, seeking to compute a velocity trajectory that maximizes the distance traveled by the ASV over the mission horizon, while adhering to the energetic constraints imposed by the limited battery capacity and available solar resource. In previous work, we utilized model predictive control (MPC)-based strategies to approximate the infinite-horizon problem \cite{govindarajan2023} via a finite-horizon online optimization framework with a terminal reward function. The terminal reward was designed to approximate the infinite-horizon benefit associated with a given state of charge at the end of the horizon. However, the approximations presented were heuristic, and tuning the terminal reward function was largely a trial-and-error process. Furthermore, the optimization relied on a prediction of the solar resource, and errors in that prediction resulted in diminished performance. Finally, the computational demands associated with this control approach are not insignificant. This becomes particularly challenging if a stochastic MPC formulation is pursued, in light of the uncertainties in the solar forecast.  

\begin{figure}
    \centering
    \includegraphics[width = \columnwidth]{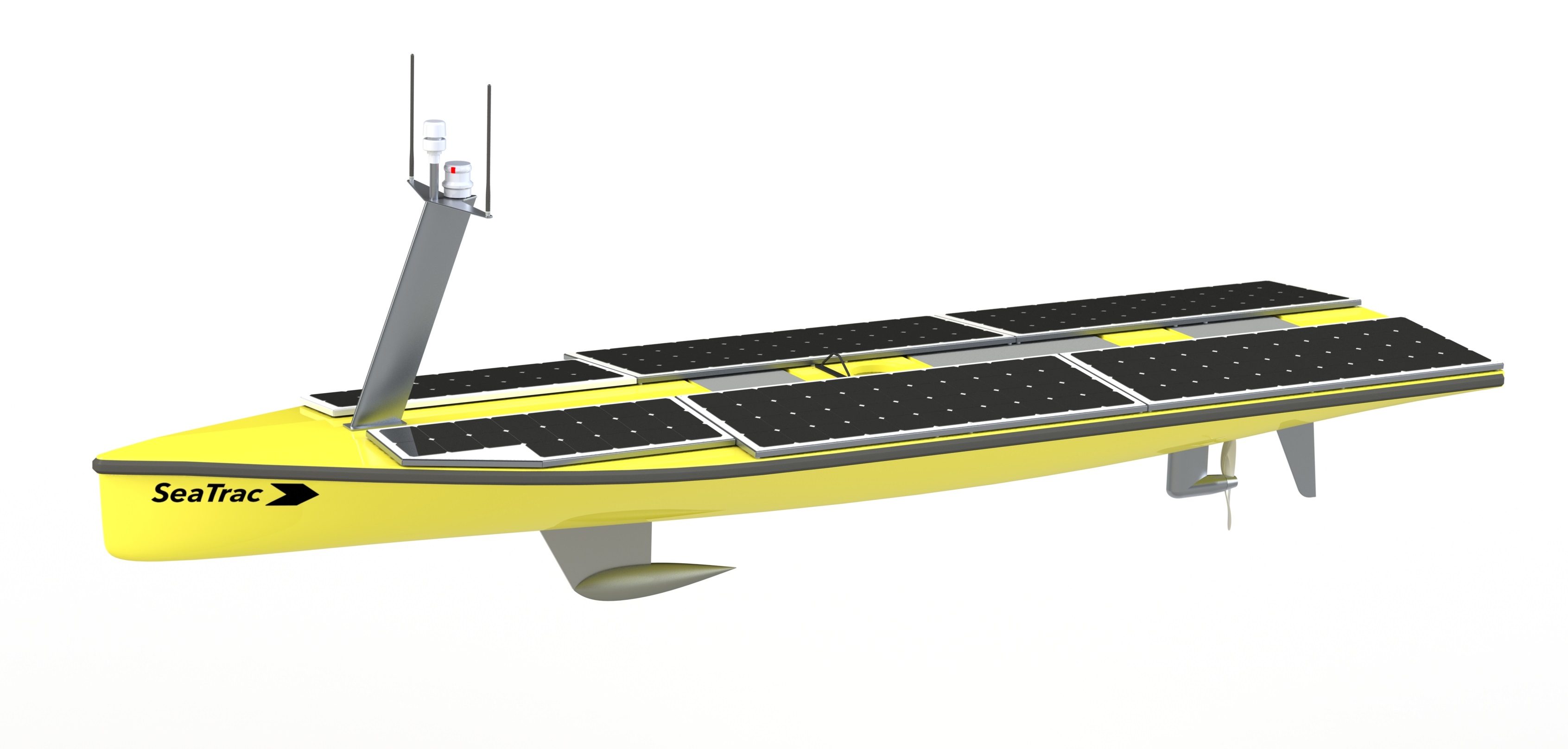}
    \caption{SeaTrac ASV considered in this work \cite{seatrac} - Image used with permission.}
    \label{fig:seatrac}
\end{figure}

To address the challenges with prior MPC-based approaches, we seek in this work to use tools from indirect optimal control to derive necessary conditions for provably optimal velocity trajectories in the presence of a time-varying solar resource and pointwise-in-time state of charge constraints. The goal of this study is twofold. First, we wish to use the insight from indirect methods to benchmark MPC-based results and better-understand the decisions being made by the MPC optimization. Second, we wish to use the results of indirect optimal control to derive simple control laws that can approximate the performance of MPC while facilitating a much simpler tuning process and significantly reduced real-time computational demands, all while eliminating the need for a forecast.

Existing literature in the transportation area, particularly on the topic of of optimal control of hybrid/electric vehicles and railway vehicles, has explored similar problems. Specifically, \cite{Miyatake2009EnergySS} has formulated \textit{finite-horizon} optimization problems subject to energetic constraints. Furthermore, an approach to solve these problems using Pontryagin's Minimum Principle (PMP) has been explored in \cite{abbas2019synthesis}. Often, the resulting necessary conditions for optimality (which manifest themselves as nonlinear mixed boundary condition differential equations) provide parametric forms for the optimal state, co-state, and control trajectories but do not lend themselves to a clean method for determining the unknown parameters. To overcome this challenge, iterative learning control (ILC) has been used to learn these parameters in the case of cyclically operating systems \cite{hu2016integrated}.

Given the cyclic nature of ASV velocity trajectory optimization that arises from diurnal patterns, the PMP/ILC approach described above serves as a sensible solution strategy. However, the existing literature does not address several challenges that we face in our problem. In particular, the existing literature only considers finite-horizon optimizations, where any cyclic element of the optimal control problem involves repeating a task under the same dynamics and same environmental conditions. The solar resource that powers the ASV, however, enables persistent missions where the resource from one day to the next can nonetheless vary. Furthermore, the time-varying environmental resource creates an issue of persistent feasibility in terms of the battery state of charge (SOC) constraints. Specifically, satisfaction of the pointwise-in-time SOC constraint at time $t$ does not guarantee feasibility at later time instances. 

In this work, we address the velocity trajectory optimization problem subject to the energetic dynamics and constraints associated with a time-varying energetic resource and long-duration mission, under an environmental resource that exhibits cyclic diurnal patters but nevertheless exhibits day-to-day variability. We first use a barrier function approach to introduce minimally tightened battery SOC constraints that guarantee persistent feasibility. We then utilize indirect optimal control methods to derive necessary conditions for optimality. These necessary conditions lead to a switching control law where the optimal velocity is piecewise constant during every time interval when the SOC constraints are not active. Because the necessary conditions for optimality do not provide the value of that optimal constant velocity (which can differ in each constraint-inactive interval), we present an ILC approach to update an estimate of that constant velocity. 

Based on a detailed model of the SeaTrac SP-48 ASV, we compare our method against several benchmarks. The first is a constant velocity command (when feasible due to SOC constraints) that ensures the energy consumed is equivalent to the energy input from the solar resource over the mission duration. We further compare this with the MPC approach developed in our earlier work in \cite{govindarajan2023}. These strategies are then all compared with a benchmark that represents the maximum achievable performance for the vehicle if it were not subject to SOC constraints. Ultimately, simulation results based on real solar data off the coast of North Carolina show that the ILC-based controller performs nearly as well as an MPC-based approach, without requiring a forecast or the associated computational burden.

In summary, the core contributions of this work include:
\begin{itemize}
    \item Minimally tightened battery SOC constraints based upon a barrier function approach to guarantee persistent feasibility;
    \item An optimal control law that is shown based on indirect optimal control methods to be piecewise constant;
    \item An iterative learning approach to estimate the value of the optimal control signal;
    \item A performance comparison to approaches such as MPC in a case-study simulation that demonstrates the comparable performance of the proposed approach without the need for an environmental forecast or the associated computational burden. 
\end{itemize}

\section{MODEL \& OPTIMIZATION PROBLEM}
\subsection{Optimal Control Formulation}
The problem of controlling the ASV velocity to maximize the distance traveled can be represented as a fixed final state, free final time optimal control problem as follows:
\begin{equation}\label{eqn:objective_function}
    \min\limits_{u(t)} J = \int_{0}^{t_f} -u(t) dt
\end{equation}
subject to the dynamics and constraints:
\begin{subequations}
    \begin{align}
        \dot{b}(t) &= P_{in}(t) - k_{h} - k_{m} u(t)^3 \label{eqn:soc_dynamics}\\
        b(t_f) &= b(0) \label{eqn:terminal_constraint}\\
        b_{min} &\leq b(t) \leq b_{max} && t \in [0, t_f] \label{eqn:state_constraint}\\
        u_{min} &\leq u(t) \leq u_{max} && t \in [0, t_f] \label{eqn:control_constraint}
    \end{align}
\end{subequations}
where $u(t)$ represents the ASV velocity, $b(t)$ represents the ASV state of charge (SOC) as a function of time, $P_{in}(t)$ represents the incoming energy (from the solar irradiance), $k_h$ represents a hotel load (constant electrical power consumption from onboard electronics and sensors), $k_m$ represents a gain that represents the power consumed by the ASV's motor, $b_{min}, b_{max}$ represent the minimum and maximum SOC for the ASV, and $u_{min}, u_{max}$ represent the minimum and maximum ASV velocity commands.

For all simulations presented in this work, we use the SeaTrac SP-48 as a reference model. The parameters associated with this vehicle are shown in Table \ref{tab:sp48}.

\begin{table}[h]
    \caption{SeaTrac SP-48 Parameters}
    \label{tab:sp48}
    \begin{center}
        \begin{tabular}{|c|c|c|c|}
        \hline
        \textbf{Variable Name} & \textbf{Symbol} & \textbf{Value} & \textbf{Units}\\
        \hline
        Hotel Load & $k_h$ & 10 & $W$\\
        Motor Constant & $k_m$ & 83 & $kg \cdot m^{-1}$\\
        Minimum State of Charge & $b_{min}$ & 0 & $Wh$\\
        Maximum State of Charge & $b_{max}$ & 6500 & $Wh$\\
        Minimum Speed & $u_{min}$ & 0 & $ m \cdot s^{-1}$\\
        Maximum Speed & $u_{max}$ & 2.315 & $ m \cdot s^{-1}$\\
        \hline
        \end{tabular}
    \end{center}
\end{table}

\subsection{Solar Irradiance Model}
We can represent the solar irradiance either with an idealized model or with real solar data. In this paper, we examine the performance of proposed control strategies under both profiles. The former profile allows us to evaluate the performance of various strategies under a relatively controlled environment, whereas the latter profile allows us to showcase the ability of our proposed methodology to perform in a realistic environment.

Our idealized solar irradiance model is described in \cite{hulstrom2003irradiance}. To summarize, the time-varying irradiance is given by:
\begin{equation}\label{eqn:irr_model}
    P_{in}(t) = \max\left(0, D_0(t) + D_1(t) \cos\left( \frac{2 \pi t}{T} \right) \right)
\end{equation}
where $D_0(t)$ represents the average level of solar irradiance, $D_1(t)$ represents the amplitude of oscillations about the average solar irradiance, and $T$ represents the period with which the solar irradiance oscillates (generally 24 hours). In this model, $D_0(t)$ and $D_1(t)$ are functions of global position, specifically latitude, and time of year. For simplicity, in this work, we fix the boat to travel along the line of latitude that intersects Cape Hatteras, North Carolina. A snapshot of the solar irradiance profile on a single day generated by the model in Eqn. \ref{eqn:irr_model} is shown in Fig. \ref{fig:model-irradiance}.

\begin{figure}
    \centering
    \includegraphics[width = \columnwidth]{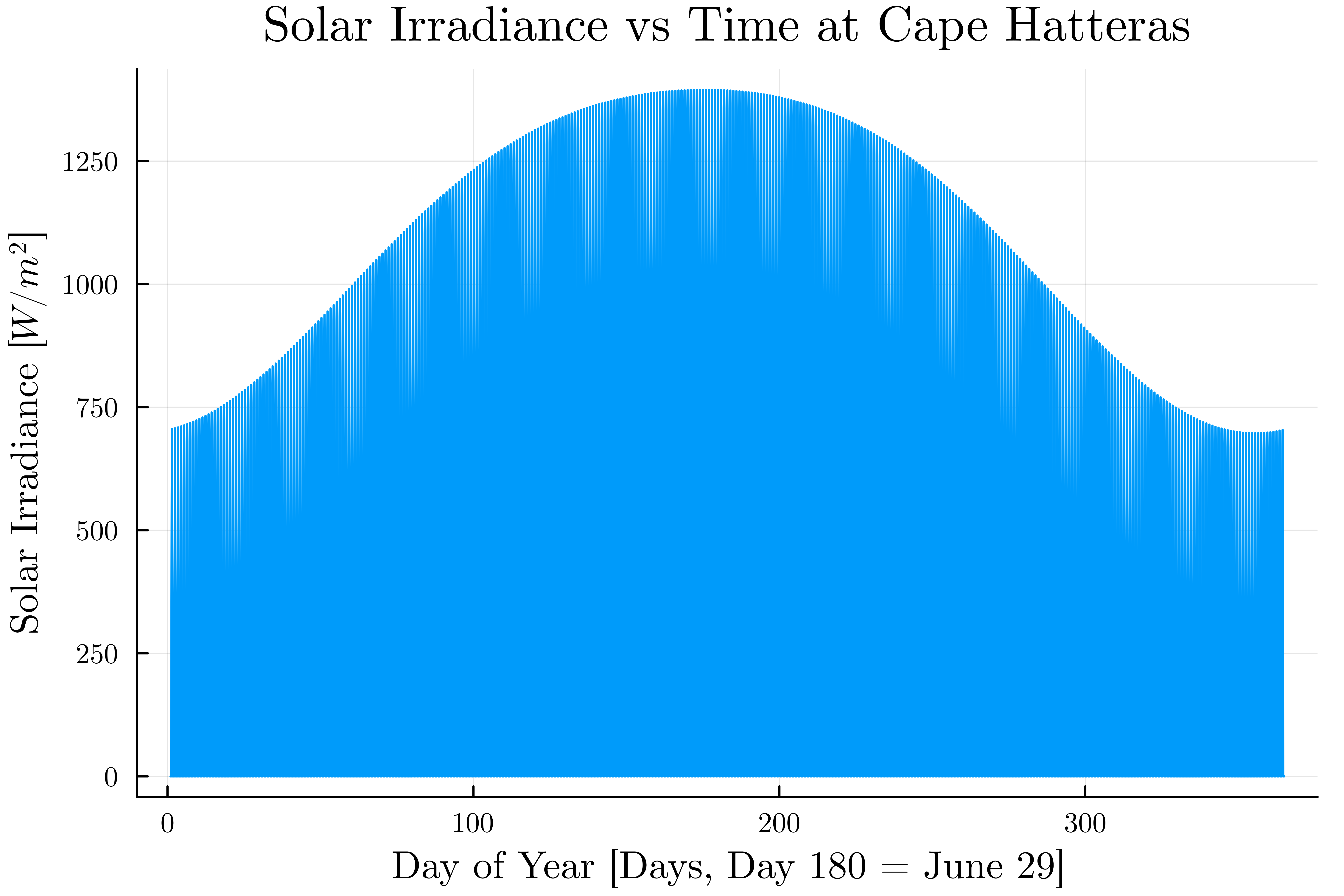}
    \caption{Idealized solar irradiance profile as modeled by Eqn. \ref{eqn:irr_model}}
    \label{fig:model-irradiance}
\end{figure}

While the aforementioned model allows for the evaluation of mission planning algorithms in a relatively controlled environment, it does not account for disturbances caused by cloud-cover or other real-world phenomena that may reduce the actual solar irradiance that reaches the surface. As such, we also utilize the solar irradiance data at Cape Hatteras in the year 2022 as provided by the ERA5 global reanalysis \cite{hersbach2020era5}. The plot of the corresponding irradiance over the year is shown in Fig. \ref{fig:2022-irradiance}.

\begin{figure}
    \centering
    \includegraphics[width = \columnwidth]{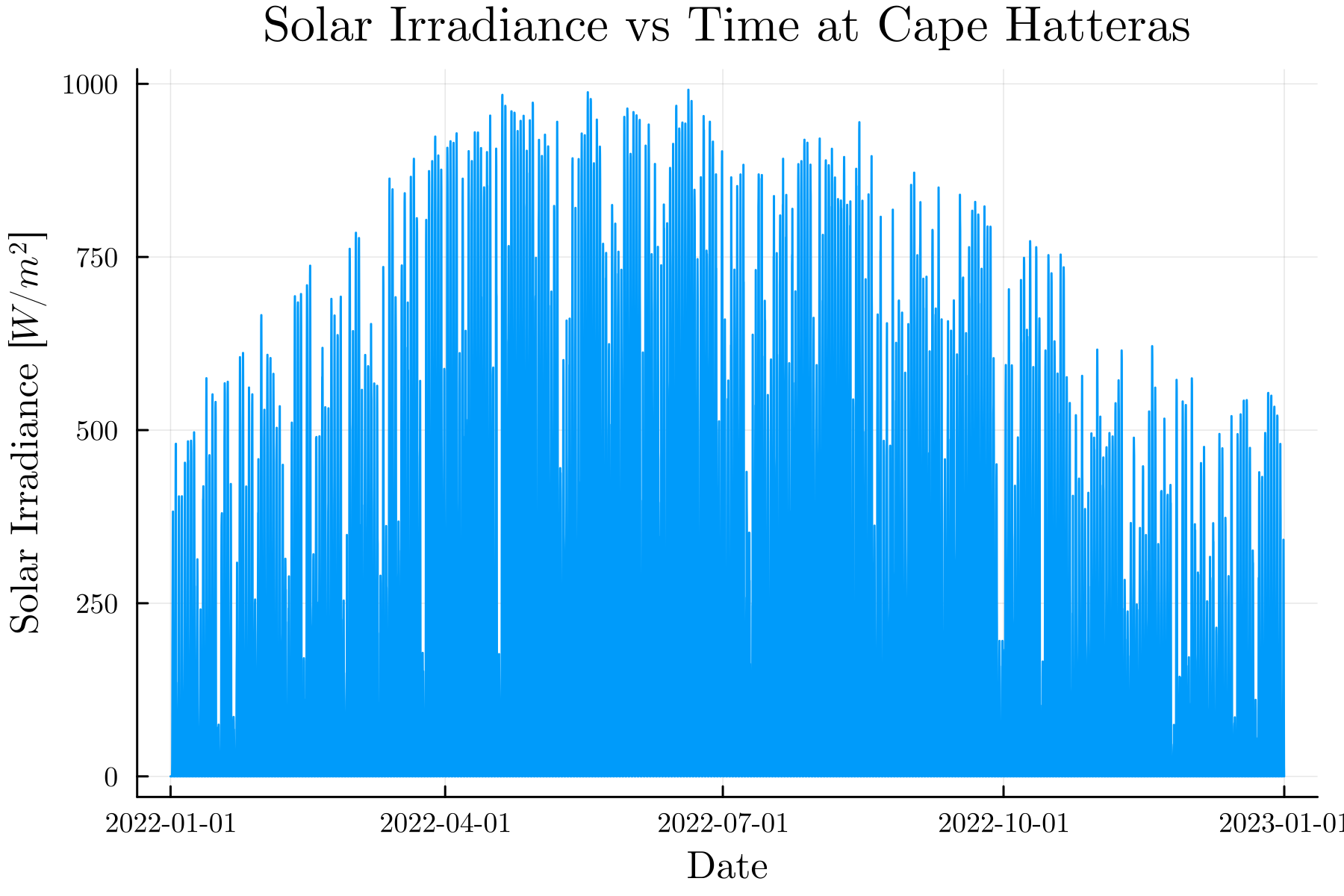}
    \caption{Actual solar irradiance at Cape Hatteras in 2022}
    \label{fig:2022-irradiance}
\end{figure}

% Metric units are preferred for use in IEEE publications in light of their
% international readership and the inherent convenience of these units in many fields.
% In particular, the use of the International System of Units (SI Units) is advocated.
%  This system includes a subsystem the MKSA units, which are based on the
%  meter, kilogram, second, and ampere. British units may be used as secondary units
%  (in parenthesis). An exception is when British units are used as identifiers in trade,
%  such as, 3.5 inch disk drive.

% \addtolength{\textheight}{-3cm}   % This command serves to balance the column lengths
%                                   % on the last page of the document manually. It shortens
%                                   % the textheight of the last page by a suitable amount.
%                                   % This command does not take effect until the next page
%                                   % so it should come on the page before the last. Make
%                                   % sure that you do not shorten the textheight too much.

%%%%%%%%%%%%%%%%%%%%%%%%%%%%%%%%%%%%%%%%%%%%%%%%%%%%%%%%%%%%%%%%%%%%%%%%%%%%%%%%
\section{OPTIMAL CONTROL FORMULATION: BARRIER FUNCTIONS \& INDIRECT METHODS}
We arrive at a simple set of equations for the optimal velocity trajectory through two steps. First, we construct barrier functions that tighten the SOC constraints by the minimal amount possible to achieve persistent feasibility. Second, we utilize indirect methods to derive necessary conditions for optimality, which show that the optimal velocity is constant during each constraint-inactive time interval.

\subsection{Barrier Functions for Persistent Feasibility}
We are subject to the dynamics and constraints shown in Eqns. \ref{eqn:soc_dynamics}-\ref{eqn:control_constraint}. However, as written, there exist circumstances where satisfaction of pointwise-in-time SOC constraints at time $t$ will not guarantee feasibility of these constraints at some future time. For example, suppose that $b(t) = b_{min}$ and the anticipated future solar resource is zero for some period of time. Then, the hotel load represented by $k_h$ will force the ASV's SOC to fall below $b_{min}$, violating our state constraints. To avoid scenarios such as these, we need to ensure that satisfaction of pointwise-in-time state constraints at any time $t$ guarantees feasibility of these constraints at any future time, which is referred to as \emph{persistent feasibility}. To do this, we will reformulate the state constraints into time-varying barrier functions, which encode this notion of persistent feasibility.

We begin by considering the lower state constraint. We do this by first defining the \emph{energy deficit}, which represents the energy consumed by the ASV assuming there exists a solar irradiance profile/charging profile $P_{in}(t)$, a hotel load $k_h$, and that $u(t) = 0$ for all $t$ in the mission. This is described mathematically by:
\begin{equation}\label{eqn:energy_deficit}
    \epsilon_{-} = \int_{t_0}^{t} (k_h - P_{in}(t)) dt
\end{equation}
\noindent where a plot of this curve is shown in Fig. \ref{fig:energy-deficit}.

\begin{figure}
    \centering
    \includegraphics[width = \columnwidth]{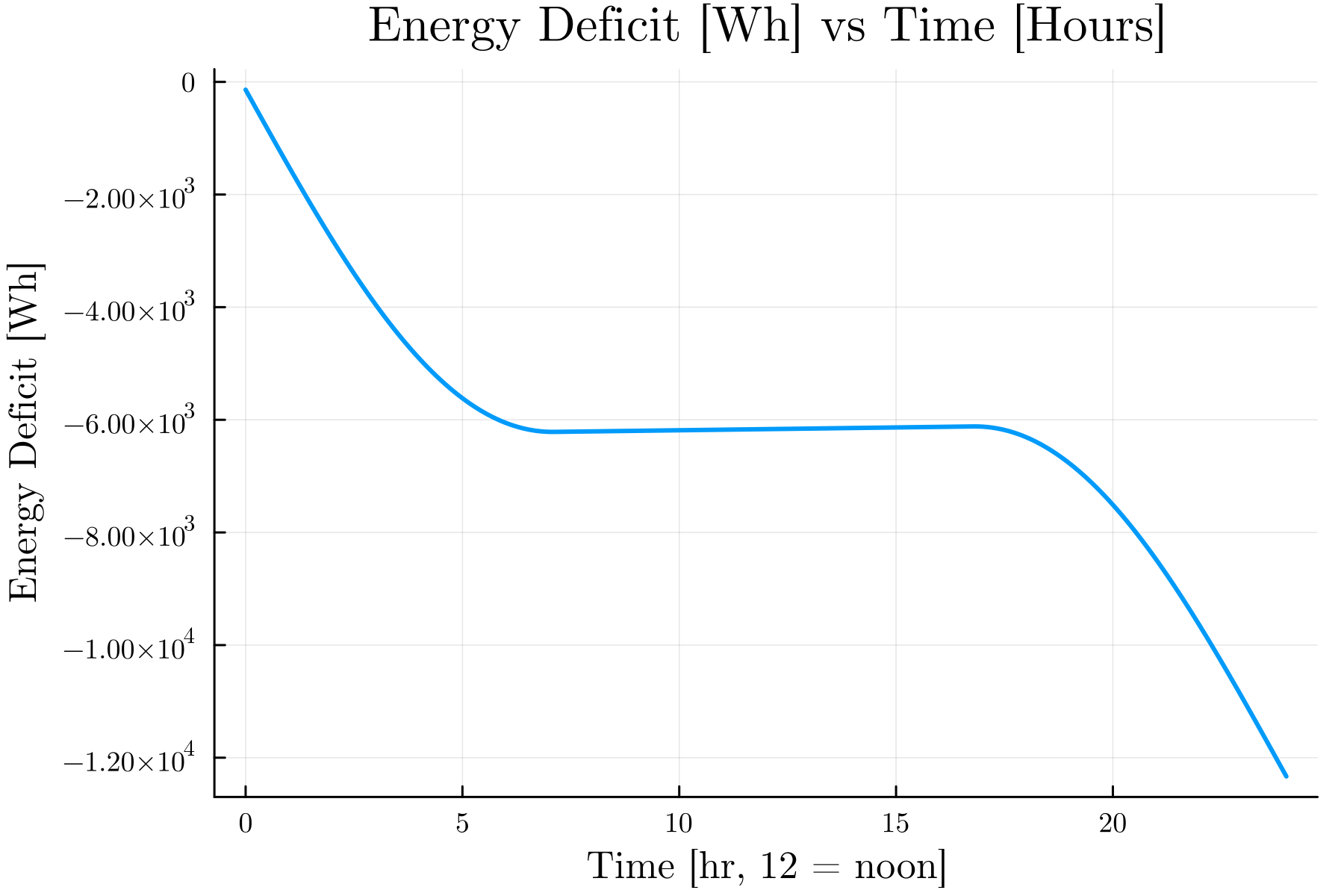}
    \caption{Energy deficit curve}
    \label{fig:energy-deficit}
\end{figure}

From this, we can then compute the lower SOC barrier by:
\begin{align}
    \epsilon_{-}^{\dag} &= \epsilon_{-}(t) - \epsilon_{-}(t_1)\\
    b_l(t_1) &= \begin{cases}
        \sup\limits_{t_2 \in [t_1, t_f]} \epsilon_{-}^{\dag} (t_2) & \sup\limits_{t_2 \in [t_1, t_f]} \epsilon_{-}^{\dag} (t_2) > 0\\
        0 & \text{ otherwise}
    \end{cases}\label{eqn:lower_soc}
\end{align}

To establish the upper SOC constraint, we can compute the \emph{energy surplus} by:
\begin{equation}\label{eqn:energy_surplus}
    \epsilon_{+} = \int_{t_0}^{t}(P_{in}(t) - k_h - k_m u_{max}^{3}(t))dt
\end{equation}

\begin{figure}
    \centering
    \includegraphics[width = \columnwidth]{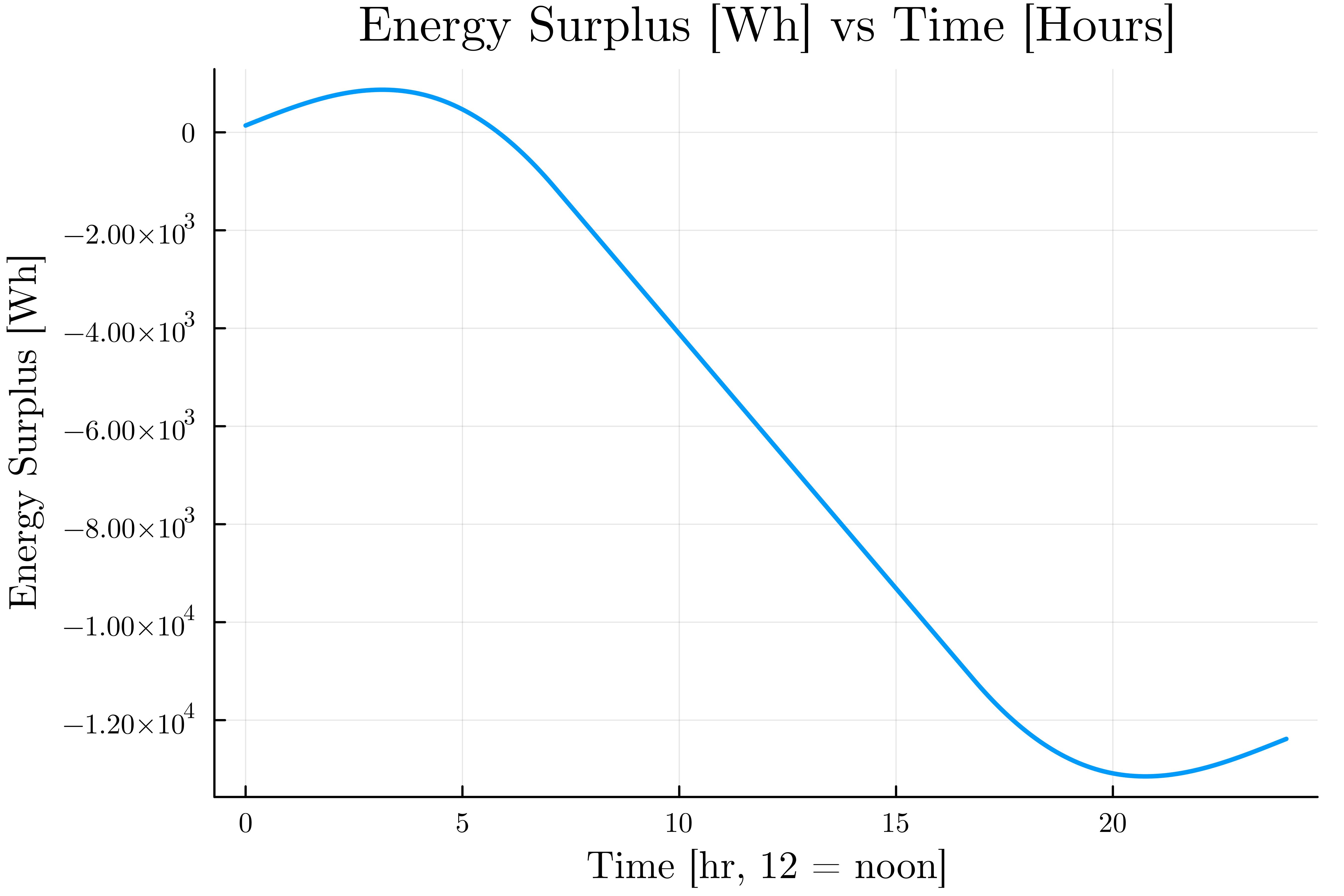}
    \caption{Energy surplus curve}
    \label{fig:energy-surplus}
\end{figure}

We can then compute the upper SOC barrier by:
\begin{align}
    \epsilon_{+}^{\dag} &= \epsilon_{+}(t) - \epsilon_{+}(t_1)\\
    b_u(t_1) &= \begin{cases}
        \sup\limits_{t_2 \in [t_1, t_f]} \epsilon_{+}^{\dag} (t_2) & \sup\limits_{t_2 \in [t_1, t_f]} \epsilon_{+}^{\dag} (t_2) > 0\\
        0 & \text{ otherwise}
    \end{cases}\label{eqn:upper_soc}
\end{align}

\noindent From Eqns. \ref{eqn:lower_soc} and \ref{eqn:upper_soc}, we can then establish SOC constraints as shown in Fig. \ref{fig:soc-barriers}.
\begin{figure}
    \centering
    \includegraphics[width = \columnwidth]{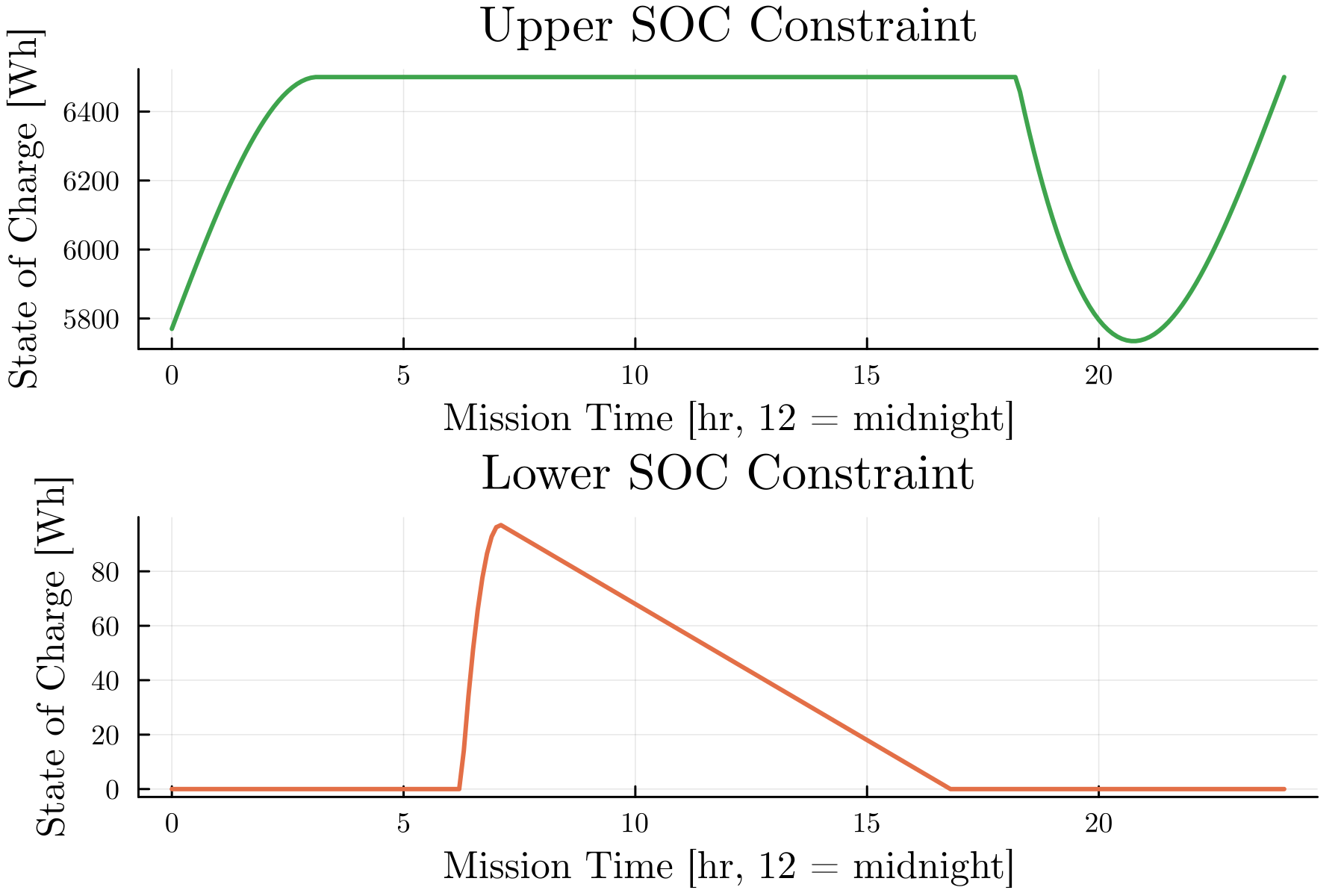}
    \caption{SOC barrier functions. These represent upper and lower limits on SOC (as functions of the time of day) that guarantee persistent feasibility.}
    \label{fig:soc-barriers}
\end{figure}

\subsection{Indirect Methods for Deriving the Optimal Velocity Profile}
Given the barrier functions introduced in the previous subsection, the original pointwise-in-time SOC constraints can be reformulated as follows:
\begin{equation}
    b_{l}(t) \leq b(t) \leq b_{u}(t),
\end{equation}
\noindent with all other constraints remaining the same as the original constraints of Eqns. \ref{eqn:soc_dynamics}-\ref{eqn:control_constraint}. With this optimal control formulation in place, it is straightforward to use indirect methods to arrive at the following Lemma:

\begin{lemma}\label{lem:constant}
    Whenever $b_{l}(t) < b(t) < b_{u}(t)$ (i.e., whenever tightened pointwise-in-time SOC constraints are inactive), the optimal velocity, $u^{*}(t)$, is constant.
\end{lemma}
\begin{proof}
From our objective function and constraints (including tightened pointwise-in-time SOC state constraints), we can first construct the following augmented state space representation, which includes an additional state variable for capturing violations of the tightened SOC constraint \cite{kirk2004optimal}:
\begin{align}
\Vec{\dot{x}} &= \begin{bmatrix}
        P_{in} - k_h - k_m u^3\\
        b^2 \mathbb{1}(-b) + [b_{max} - b]^2 \mathbb{1}(b_{max} - b)
    \end{bmatrix}\label{eqn:states}
\end{align}

From here, we can construct the following Hamiltonian:
\begin{align}\label{eqn:hamiltonian}
    H(t,x(t),u(t),p(t)) = \\ 
    -u(t) + p_1(t)(P_{in}(t) - k_h 
     - k_m u(t)^3) \nonumber \\ + 
        p_2(t)((b_{u}(t)
        - b(t))^2 \mathbb{1}(b(t) - b_{u}(t)) \nonumber  \\
        + (b(t) - b_{l}(t))^2 \mathbb{1}(b_{l}(t)-b(t))) \nonumber
\end{align}
where $p_1(t)$ and $p_2(t)$ represent the co-states and $\mathbb{1}(x)$ represents the Heaviside step function for a generic argument $x$.

This results in the following co-state equations:
\begin{eqnarray}
    \dot{p}_{1}(t) &=& -2p_2(t)(b_{u}(t) - b(t))\mathbb{1}(b(t) - b_{u}(t)) \\ \nonumber && + p_2(t) (b_{u}(t) - b(t))^2 \delta(b(t) - b_u(t)) \\ \nonumber && + 2 p_2(t) (b(t) - b_l(t)) \mathbb{1}(b_l(t) - b(t)) \\ \nonumber && + p_2(b(t) - b_l(t))^2 \delta(b_l(t) - b(t))\\
    \dot{p}_{2}(t) &=& 0
\label{eqn:costates}
\end{eqnarray}
\noindent where $\delta(x)$ represents the unit impulse function for a generic argument $x$.

Note that when the state constraints are satisfied, $\dot{x}_2 = 0$ (consequently, $x_{2} = 0$ for all time). Furthermore, when pointwise-in-time tightened SOC constraints are inactive (i.e. $b_{l}(t) < b(t) < b_{u}(t)$), all of the Heaviside and impulse functions evaluate to zero and $\dot{p}_1 = 0$ (consequently, $p_{1}$ is constant). Accounting for these facts and noting that $\frac{\partial H}{\partial u} = 0$ at the optimum whenever $b_{l}(t) < b(t) < b_{u}(t)$, the optimal control solution is given by:
\begin{equation}\label{eqn:control}
    u^{*} = \begin{cases}
        u_{max} & b \geq b_{u}(t)\\
        u_{min} & b \leq b_{l}(t)\\
        \sqrt{\frac{-1}{3 k_m p_1}} & \text{otherwise}\\
    \end{cases}
\end{equation}

\noindent As we previously established, when the state constraints are satisfied, $\dot{p}_1 = 0$, which implies that $p_1$ is constant. Therefore, $u^{*}$ when unconstrained must be a constant, thus proving Lemma \ref{lem:constant}.
\end{proof}

Eqn. \ref{eqn:control} and Lemma \ref{lem:constant} provide a simple, yet useful result for the control of the ASV. However, the indirect methods used in the proof do not provide a straightforward mechanism for solving for $p_1$, which can vary between consecutive constraint-inactive intervals. An ILC-based method for estimating $p_1$ is detailed in the next section.

\section{REAL-TIME REALIZATION OF THE OPTIMAL VELOCITY PROFILE THROUGH ITERATIVE LEARNING}
The previous section has established that the optimal velocity control strategy is a switching one, where the optimal velocity is constant during each constraint-inactive interval. Two complications arise from the aforementioned analysis. First, the actual controller must be implemented in sampled (discrete) time, in the presence of sensor noise. In such a scenario, a switching controller is prone to significant chatter, which must be mitigated for the resulting controller to be practical. Secondly, while the previously derived theory establishes that the optimal velocity is constant during each constraint-inactive interval, it does not establish the optimal value of that constant velocity. This section addresses both of these complications.

\subsection{Discretization}
When we discretize time, the switching controller will chatter as the SOC barrier functions go from active to inactive (and vice versa). This chatter will be further complicated in actual experimental implementations where the SOC measurement is corrupted by noise. To address these complications, we modify the control solution to include the ``buffer'' shown in Eqn. \ref{eqn:buffer}:

\begin{equation}\label{eqn:buffer}
    u(t) =  \begin{cases}
        \frac{b(t) - b_{l}(t)}{\delta} u^{*}(t) + \left(1 - \frac{b(t) - b_l(t)}{\delta}\right) u_{min} \\ \hspace{2cm}\text{when }0 < b(t) - b_l(t) < \delta\\
        \frac{b_{u}(t) - b(t)}{\delta} u^{*}(t) + \left(1 - \frac{b_{u}(t) - b(t)}{\delta}\right) u_{max} \\ \hspace{2cm}\text{when }0 < b_{u}(t) - b(t) < \delta\\
    \end{cases}
\end{equation}
where $\delta$ is a tunable parameter that represents the width of the buffer in Watt-hours and $u^{*}$ represents the optimal unconstrained velocity.

\subsection{Iterative Learning}
\begin{figure}
    \centering
    \includegraphics[width=0.75\linewidth]{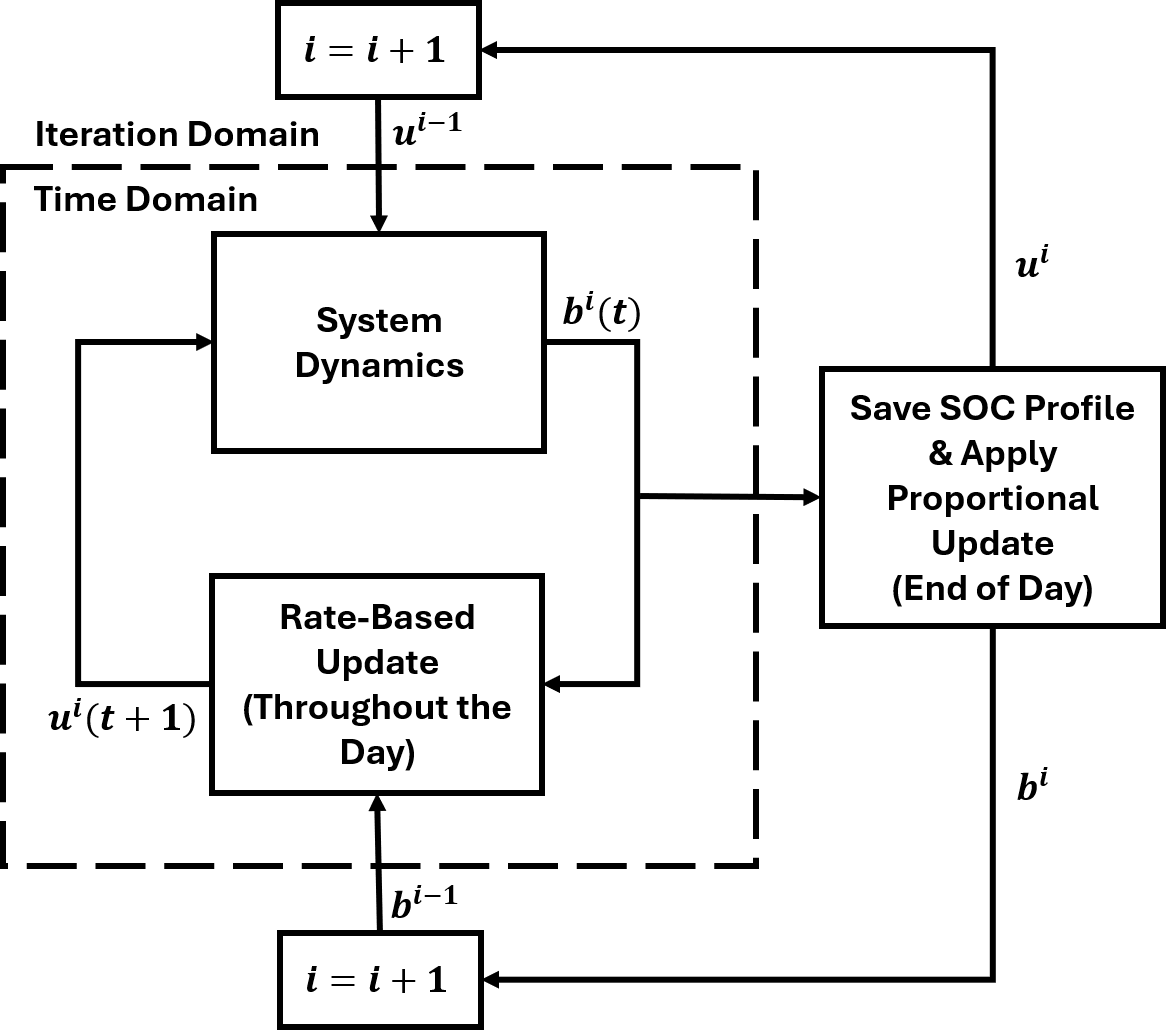}
    \caption{Block diagram demonstrating the iterative approach to learning the optimal constant speed.}
    \label{fig:ilc_block}
\end{figure}
To compute the optimal value of $p_1$ (which is known to be constant during each constraint-inactive interval but otherwise unknown), and thus the value of the commanded velocity, we utilize an iterative learning approach as depicted in Fig.\ref{fig:ilc_block}. To do this, we begin with an initial guess of $p_1$, then update the estimate for this co-state value as per the update laws defined in Eqns. \ref{eqn:proportional_update} and \ref{eqn:rate_update}. We define each iteration $i$ to be a 24-hour period (one solar cycle). The update law is given in two parts. The first of these parts is a once-per-iteration update law given by:
\begin{equation}\label{eqn:proportional_update}
    u^{i+1} = u^{i} + k_p (b^{i}(t_f) - b_{des}^{i})
\end{equation}
where $k_p$ is a proportional learning gain that acts upon the difference between the SOC at the end of iteration  (represented by $b^{i}(t_f)$) and the desired terminal state of charge ($b_{des}^{i}$). This update is applied at the end of each iteration (day).

The second part of the update law represents a continuous rate-based update law that is given by:
\begin{equation}\label{eqn:rate_update}
    u^{i}(t+1) = u^{i} + k_d (b^{i}(t) - b^{i-1}(t))
\end{equation}
where $k_d$ is a rate-based gain that acts upon the difference between the current SOC and the SOC at the same time in the previous iteration (day). This update is then applied to the velocity from the same time in the previous iteration. This update law is applied continually following the first iteration and serves as a ``damping'' term on the control updates.

To demonstrate the convergence of the above strategy to an optimal costate value $p_1$, we apply the strategy to a simulation where the environmental conditions are repeated across each iteration. The results of this simulation are shown in Fig. \ref{fig:costate_v_iteration}, and the velocities corresponding to the costate values are shown in Fig.\ref{fig:velocity_v_iteration}. From these figures, we can see that the costate converges to the optimal value in under 15 iterations for $k_p = 5 \times 10^{-5}$ and $k_d = 1 \times 10^{-5}$. This optimal costate value is $p_1 = -0.0012$, which corresponds to a speed of $u = 1.83$ $m \cdot s^{-1}$.

\begin{figure}
    \centering
    \includegraphics[width = \columnwidth]{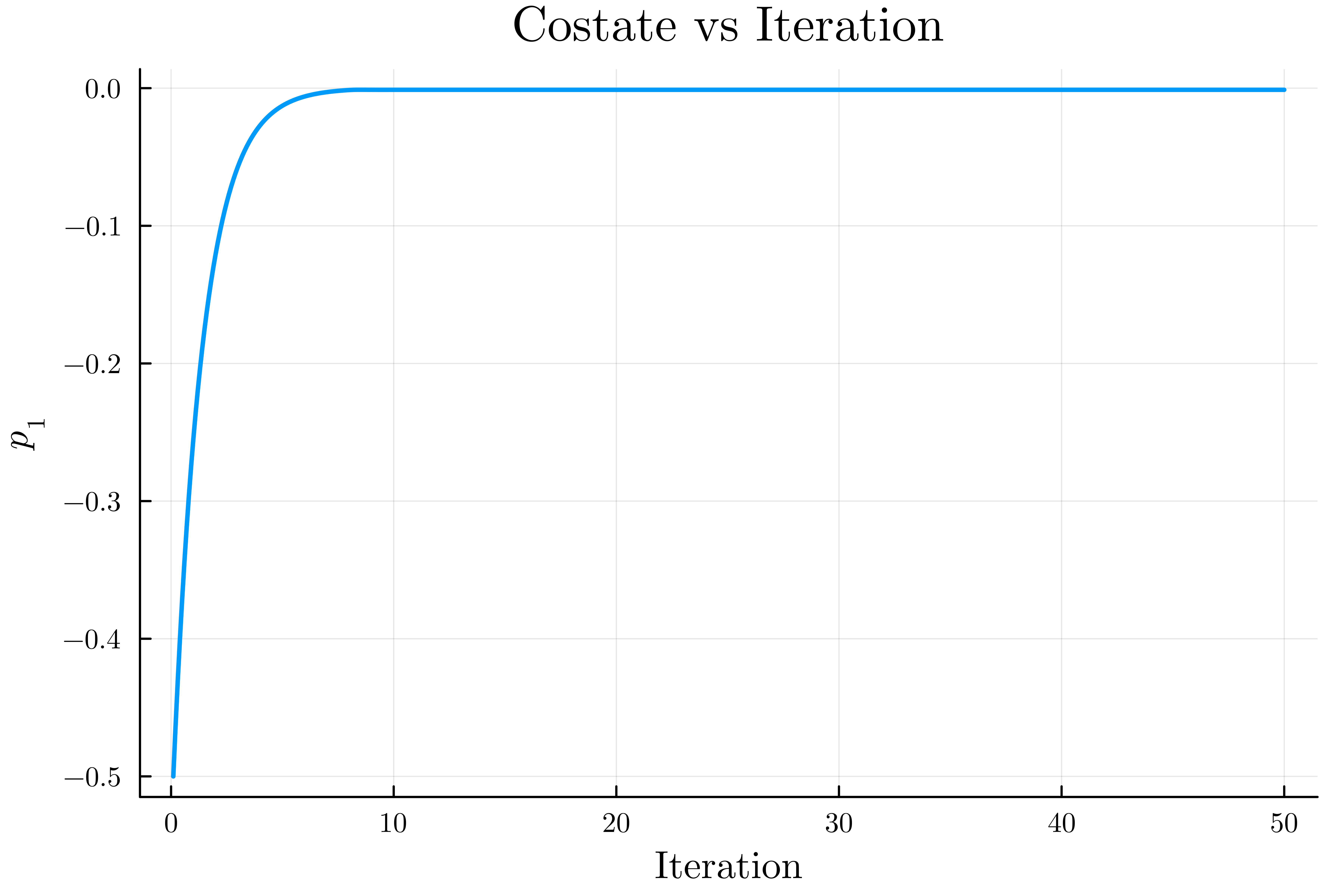}
    \caption{Costate vs iteration under a consistent day-to-day solar irradiance profile. This simulation demonstrates convergence of the ILC algorithm under consistent solar conditions.}
    \label{fig:costate_v_iteration}
\end{figure}

\begin{figure}
    \centering
    \includegraphics[width = \columnwidth]{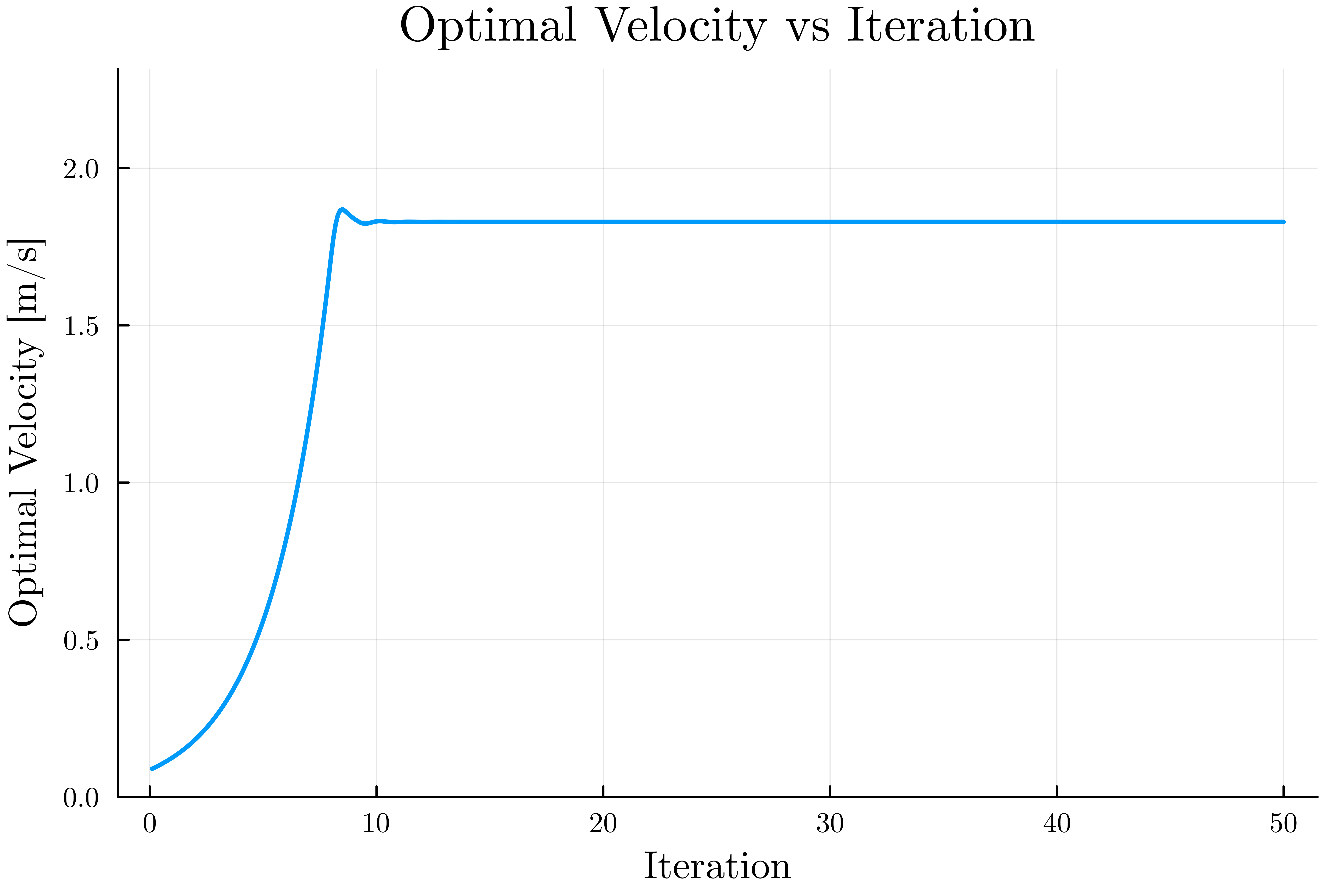}
    \caption{Optimal Velocity vs iteration under a consistent day-to-day solar irradiance profile. This simulation demonstrates convergence of the ILC algorithm under consistent solar conditions.}
    \label{fig:velocity_v_iteration}
\end{figure}

In the following section, we demonstrate the performance of the proposed control strategy against other comparison strategies in simulation.

\section{RESULTS}
To benchmark the performance of the proposed controller, we compare it against the performance of other strategies in simulation.

\subsection{Comparison Strategies}
As an initial basis for comparison, we compute the constant velocity for which the energy expended over a one-year simulation duration is exactly equal to the solar energy available. To do this, we solve the below equation for $u$:
\begin{equation}
    \int_{0}^{t_f} k_m u^3 dt = \int_{0}^{t_f} P_{in}(t) dt
\end{equation}
where $k_m$ represents a gain that represents the power consumed by the ASV's motor, $u$ is the constant velocity of the ASV, and $P_{in}(t)$ is the energy produced by the ASV from the solar resource from times $0$ to $t_f$. We allow the ASV to travel at this velocity without being subject to the SOC constraints to establish the upper limit on ASV performance. Note that this is not in fact a tight upper bound, as the ASV will be subject to SOC constraints in reality.

As a second comparison control strategy, we implement the aforementioned constant-velocity strategy only when the tightened SOC constraints are inactive. During periods in which the previously computed constant velocity would lead to violation of tightened SOC constraints, the velocity is set to the lower or upper limit in order to satisfy the constraint at equality.

Finally, our indirect methods + ILC strategy and the two comparison strategies described above are compared with an MPC strategy as implemented in \cite{govindarajan2023}.

To evaluate these strategies, we considered two simulation scenarios, where we seek to maximize the distance traveled by the ASV over a mission period of one year. In the first scenario, the idealized solar irradiance model was used as the energetic resource. In the second scenario, ERA5 data for Cape Hatteras was used \cite{hersbach2020era5}.

\subsection{Simulation Results}

For the simulations performed under the idealized solar resource profile, from Fig. \ref{fig:dist-plot}, we can see that the proposed learning controller outperforms the constant velocity benchmarks, while nearly matching the performance of MPC. It is noteworthy that the MPC implementation is based on a solar resource forecast, and this forecast is assumed to be perfect for the purpose of optimization. The ILC-based approach does not require any forecast and performs almost identically under the idealized solar resource model. The daily-averaged velocity and SOC are plotted against time in Figs. \ref{fig:vel-average} and \ref{fig:soc-average}. Figs. \ref{fig:vel-average} and \ref{fig:soc-average} illustrate an increased level of volatility in the ILC approach, as ILC attempts to learn the optimal velocity. This volatility, however, is not associated with any appreciable performance reduction.

Similar results can be seen when the algorithms are applied to the real-world solar irradiance profile. From Fig. \ref{fig:2022-vel-average}, we can see that the ILC-based strategy is particularly sensitive to the fluctuations in solar irradiance that are present in the real-world data. This is attributable to the fact that the estimated optimal velocity value requires time to adjust to day-to-day variations in solar irradiance due to atmospheric conditions, whereas the MPC strategy has the benefit of accessing a perfect forecast. Accordingly, the state of charge is particularly volatile as it reacts to the changing velocity and solar profile as seen in Fig.\ref{fig:2022-soc-average}. Nevertheless, the ILC approach proposed in this work nearly matches the MPC performance, even in the presence of real-world solar fluctuations.

The ILC approach falls just short of the performance levels yielded by MPC, yet it possesses tremendous advantages. First, it requires no forecast of the solar resource. Secondly, the ILC approach does not require the heuristic tuning of a terminal reward function within a real-time optimization. Finally, the learning controller is simpler to implement and requires less computational power. This is due to the fact that the learning controller only has two gains that must be tuned for performance, whereas the MPC requires a detailed model. 

The presented strategies were implemented in Julia \cite{bezanson2017julia} and simulated on a computer with an Intel\textregistered Xeon\textregistered W-2125 4.0 GHZ CPU and 32GB of RAM. For the same year-long simulation at a time step of 6 minutes, the ILC approach took roughly 20 seconds to simulate, whereas the MPC strategy (implemented using the JuMP Ipopt optimizer \cite{DunningHuchetteLubin2017}) took 21 hours and 36 minutes. This demonstrates a significant improvement in computational efficiency using the presented ILC approach.

\begin{figure}
    \centering
    \includegraphics[width = \columnwidth]{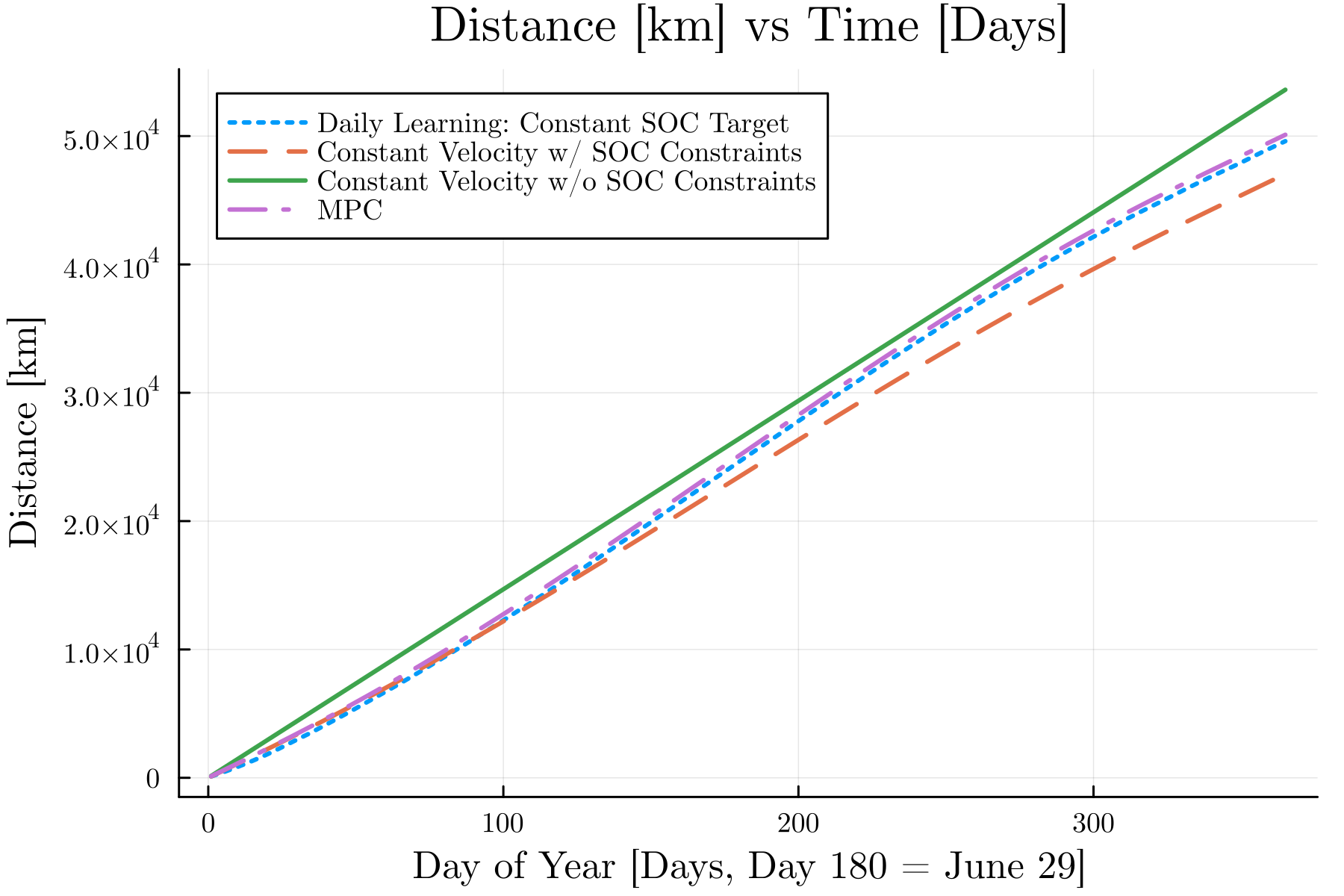}
    \caption{Distance traveled vs. time under the idealized solar irradiance model.}
    \label{fig:dist-plot}
\end{figure}

\begin{figure}
    \centering
    \includegraphics[width = \columnwidth]{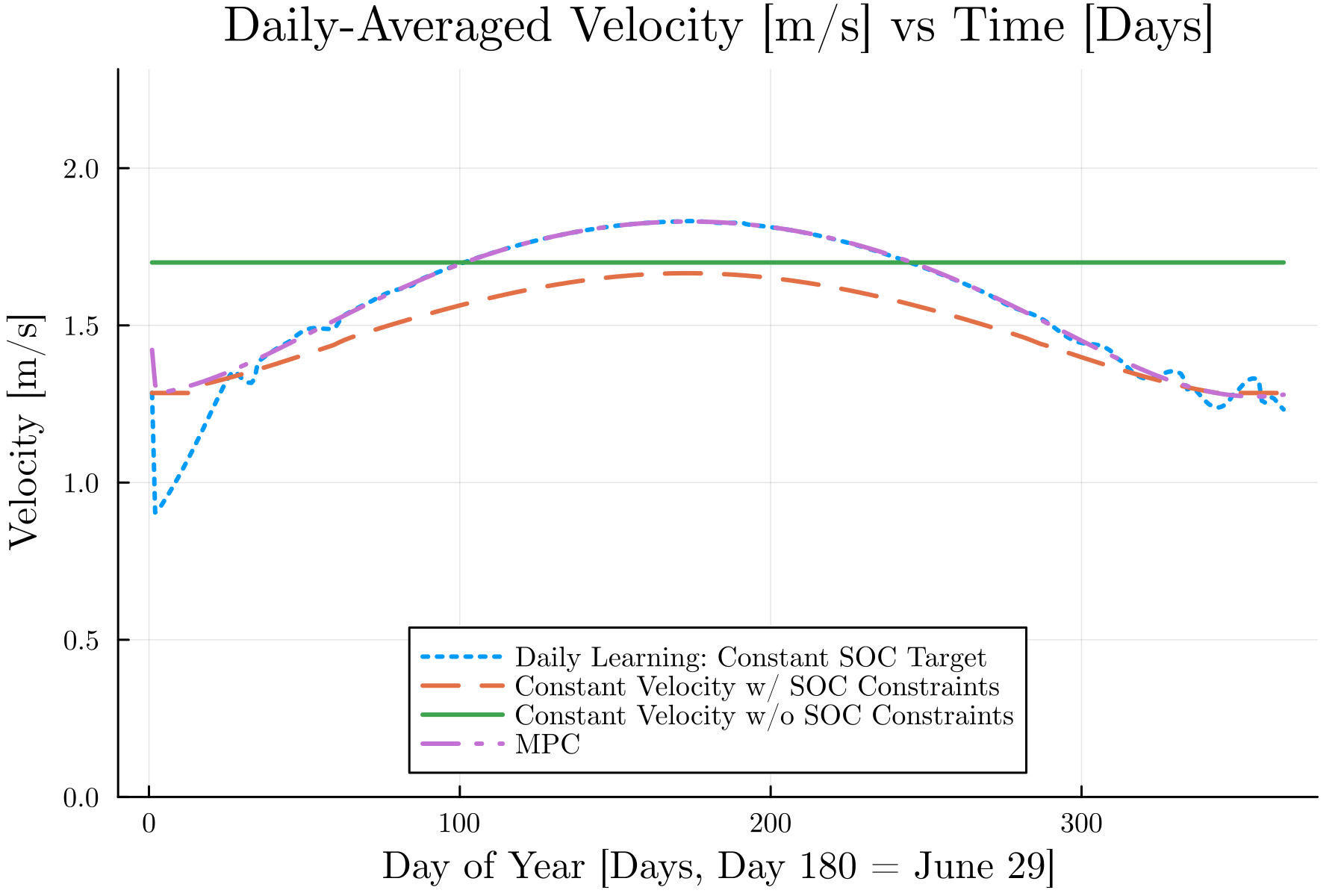}
    \caption{Velocity vs. time under the idealized solar irradiance model.}
    \label{fig:vel-average}
\end{figure}

\begin{figure}
    \centering
    \includegraphics[width = \columnwidth]{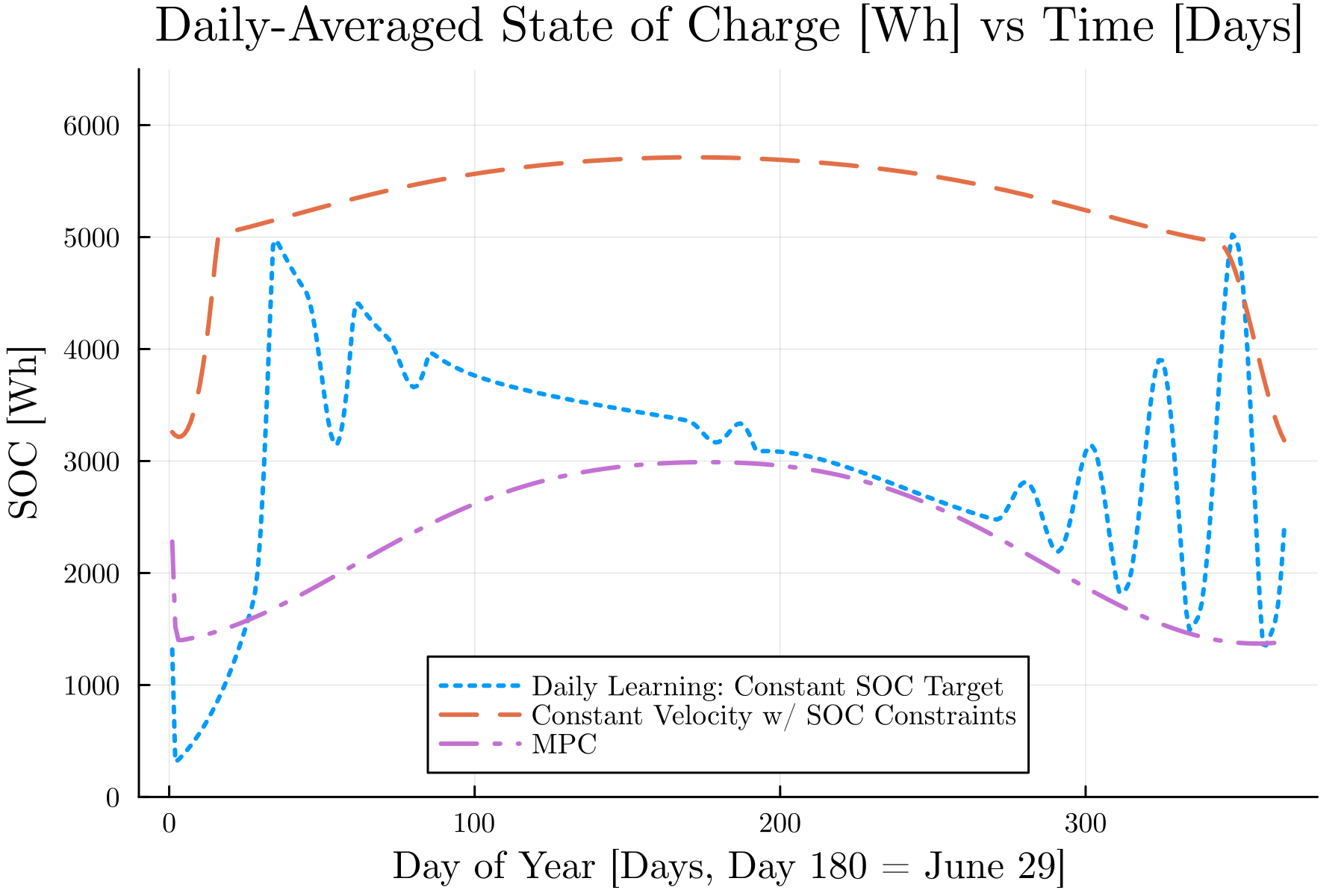}
    \caption{SOC vs. time under the idealized solar irradiance model.}
    \label{fig:soc-average}
\end{figure}

\begin{figure}
    \centering
    \includegraphics[width = \columnwidth]{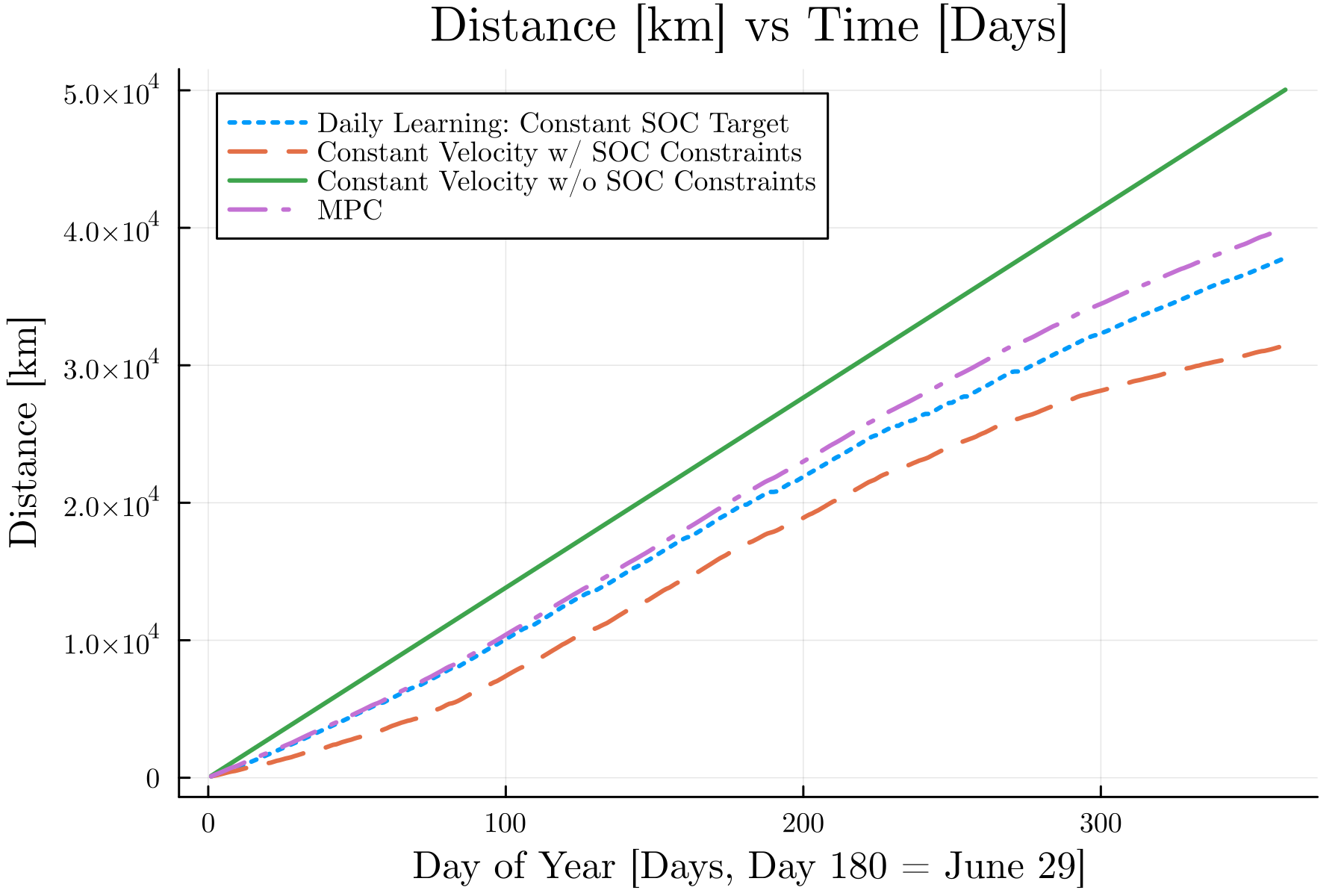}
    \caption{Distance traveled vs. time under the actual solar irradiance data for Cape Hatteras.}
    \label{fig:2022-dist-plot}
\end{figure}

\begin{figure}
    \centering
    \includegraphics[width = \columnwidth]{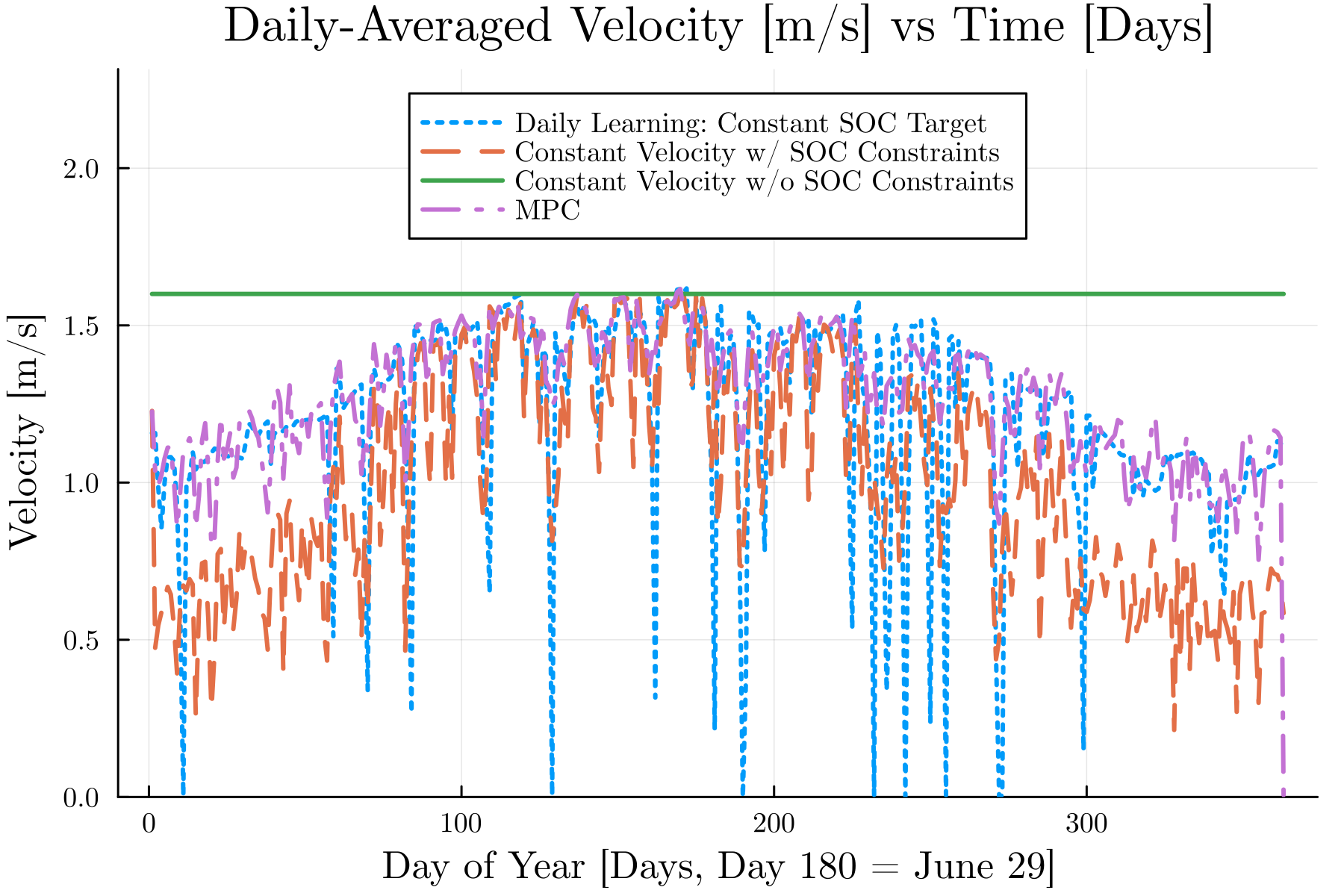}
    \caption{Velocity vs. time under the actual solar irradiance data for Cape Hatteras.}
    \label{fig:2022-vel-average}
\end{figure}

\begin{figure}
    \centering
    \includegraphics[width = \columnwidth]{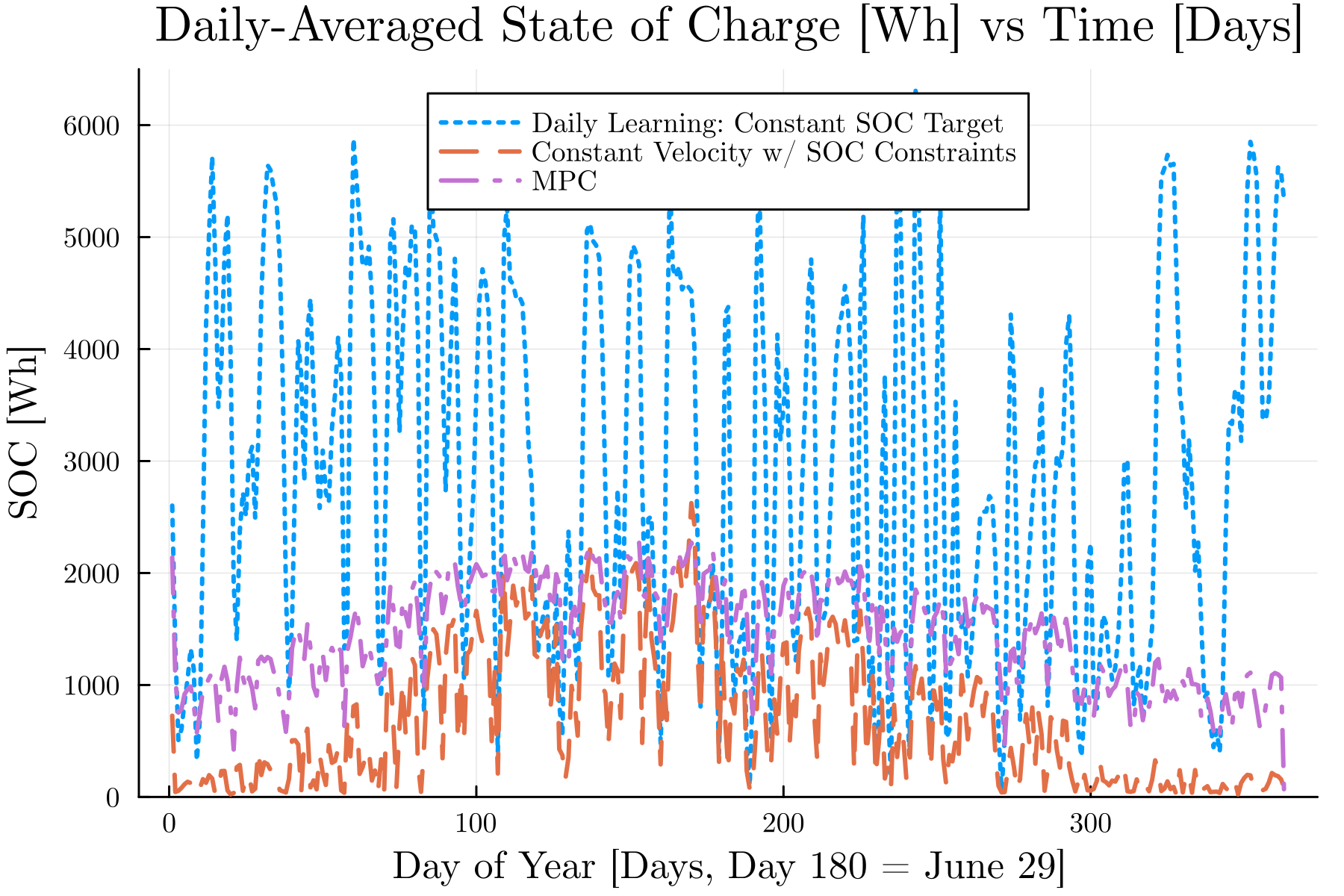}
    \caption{SOC vs. time under the actual solar irradiance data for Cape Hatteras.}
    \label{fig:2022-soc-average}
\end{figure}

%%%%%%%%%%%%%%%%%%%%%%%%%%%%%%%%%%%%%%%%%%%%%%%%%%%%%%%%%%%%%%%%%%%%%%%%%%%%%%%%
\section{CONCLUSIONS AND FUTURE WORK}
In this work, we considered the problem of long-horizon or persistent velocity trajectory optimization for a solar-powered autonomous surface vessel (ASV). Because of the unique features of this system that do not immediately guarantee persistent feasibility of pointwise-in-time state constraints, we constructed barrier functions to tighten these constraints by the minimum amount required to ensure persistent feasibility. We then utilized indirect methods to arrive at a provably optimal switching control law, where we subsequently used iterative learning control (ILC) to obtain estimates of the optimal constant velocity during each constraint-inactive interval. Using real solar data and a model of the SeaTrac SP-48 ASV, we showed that this formulation outperforms simple benchmark strategies. Furthermore, we showed that the proposed approach nearly matches the performance of a more complex MPC strategy that requires a solar forecast.

In future work, we will focus on replacing our distance-maximization objective with an actual information-maximization mechanism. Furthermore, we will focus on combining velocity trajectory optimization with path planning, all with the goal of maximizing information gathered. Finally, we will focus on experimentally validating the developed control algorithms on the SP-48 ASV presently in our possession.

%%%%%%%%%%%%%%%%%%%%%%%%%%%%%%%%%%%%%%%%%%%%%%%%%%%%%%%%%%%%%%%%%%%%%%%%%%%%%%%%
% \section{ACKNOWLEDGMENTS}

% The authors gratefully acknowledge the contribution of National Research Organization and reviewers' comments.

%%%%%%%%%%%%%%%%%%%%%%%%%%%%%%%%%%%%%%%%%%%%%%%%%%%%%%%%%%%%%%%%%%%%%%%%%%%%%%%%
% \begin{thebibliography}{99}

% \bibitem{c1}
% J.G.F. Francis, The QR Transformation I, {\it Comput. J.}, vol. 4, 1961, pp 265-271.

% \bibitem{c2}
% H. Kwakernaak and R. Sivan, {\it Modern Signals and Systems}, Prentice Hall, Englewood Cliffs, NJ; 1991.

% \bibitem{c3}
% D. Boley and R. Maier, "A Parallel QR Algorithm for the Non-Symmetric Eigenvalue Algorithm", {\it in Third SIAM Conference on Applied Linear Algebra}, Madison, WI, 1988, pp. A20.

% \end{thebibliography}
\bibliographystyle{IEEEtran}
\bibliography{references}

\end{document}